\documentclass[preprint,12pt]{elsarticle}
\usepackage[utf8]{inputenc}
\usepackage{lineno,hyperref}
\usepackage{array}
\usepackage{amssymb}
\usepackage{enumerate}
\usepackage{makecell}
\usepackage{mathtools}
\usepackage{array}
\usepackage{multirow}
\usepackage{slashed}
\usepackage{ulem}
\usepackage{color}
\usepackage{diagbox}
\usepackage{enumerate}
\usepackage{graphicx}
\allowdisplaybreaks[1]
\usepackage{subfigure}
 \usepackage{tikz} 
\usepackage{indentfirst}
\usepackage{soul}
\usepackage{amsmath}
\usepackage[justification=centering]{caption}
\usepackage{tabularx}
 \usepackage{algorithm}
 \usepackage{algorithmicx}
\usepackage{algpseudocode}
\usepackage{adjustbox}
\usepackage{geometry}
\usetikzlibrary{shapes.geometric, arrows}
\modulolinenumbers[5]
\newtheorem{definition}{Definition}[subsection]
\newtheorem{example}{Example}[subsection]
\newtheorem{proposition}{Proposition}[subsection]

\newtheorem{proof}{Proof}[subsection]
\newtheorem{remark}{Remark}

\newcommand{\circledsmall}[1]{\hbox{\tikz\draw (0pt, 0pt)
    circle (.5em) node {\makebox[0.15em][c]{\scriptsize#1}};}}
\newcommand{\circledtiny}[1]{\hbox{\tikz\draw (0pt, 0pt)
    circle (.4em) node {\makebox[0.01em][c]{\tiny#1}};}}
    \newcommand{\circledlarge}[1]{\hbox{\tikz\draw (0pt, 0pt)
    circle (.6em) node {\makebox[0.15em][c]{\small#1}};}}

\begin{document}

\begin{frontmatter}

\title{Isopignistic Canonical Decomposition via Belief Evolution Network}

\author[inst1]{Qianli Zhou}
\author[inst1]{Tianxiang Zhan}
\author[inst1,inst5]{Yong Deng}

\affiliation[inst1]{organization={Institute of Fundamental and Frontier Science},
organization={University of Electronic Science and Technology of China},%Department and Organization
            city={Chengdu},
            postcode={250014},
            country={China}}

\affiliation[inst5]{organization={School of Medicine},
organization={Vanderbilt University},
            city={Nashville},
            postcode={37240},
            country={USA}}

\begin{abstract}
We propose the isopignistic canonical decomposition for belief functions, offering a novel perspective that unifies probability theory, possibility theory, and Dempster-Shafer theory. Firstly, an isopignistic transformation method is proposed, which can transfer a certain belief function to arbitrary positions within its corresponding isopignistic domain. We then develop it into a new canonical decomposition for belief functions, decomposing the belief function into a propensity component and a commitment component.The boundary cases of the commitment component correspond to unique possibility and probability distributions, respectively. More generally, we discuss the advantages of this canonical decomposition in terms of both uncertainty measurement and information fusion, and utilize it to refine the hyper-cautious transferable belief model.
\end{abstract}


\end{frontmatter}

\section{Introduction}
Dempster-Shafer (DS) theory of evidence, also known as belief function theory, is an effective artificial intelligence tool for modeling and handling uncertainty in partial knowledge environments. The DS theory is derived from probability multi-valued mapping \cite{dempster2008upper} and is formalized as a random coding semantics \cite{shafer1976mathematical}. By assigning the mass to the frame's power set, the basic probability assignment (BPA) generalizes the probability distribution. Non-refinable beliefs are assigned on the multi-element propositions. Upon receiving enough knowledge, these beliefs can be transferred into singletons, and the BPA equals a probability distribution. Hence, the BPA can be viewed as a probability distribution that introduces ignorance through the multi-valued mapping. Possibility theory is an uncertainty information theory developed from fuzzy set theory, which uses the membership function to characterize the degree of non-negativity of elements within the frame \cite{solaiman2019possibility}. In DS theory, the plausibility function, which represents identical information content with BPA, is also used to indicate the degree of non-negativity of elements within a proposition. For a basic probability assignment (BPA) with nested focal sets, called a consonant mass function, its plausibility function is identical to the possibility function. Hence, a widely accepted bridge between DS theory and possibility theory is the consonant mass function \cite{dubois2001new,dubois2008definition}. When possibility and probability distributions are used to represent restrictions on the same uncertain variable, they can be connected via two BPAs. However, the relationship of general BPAs to them and the corresponding transformation method is an important but neglected topic, which is the first motivation of this paper.

Handling uncertainty in a reasonable manner is also a key step for reasoning and decision-making. When there are available BPAs from multiple independent and reliable sources, Dempster's rule can effectively fuse them, forming the basis of evidential reasoning and the generalized Bayes theorem. Due to the incomplete mutual exclusivity among focal sets, BPA can not visualise and process the information sufficiently. Therefore, some identical information content representations are proposed to realize the efficient and diverse combination rules. Among them, the best known are belief, plausibility and commonality functions based on the definition of set operations, which can implement conjunctive and disjunctive rules efficiently. In addition, the canonical decomposition is an another effective path to implement identical information content representation. Smets \cite{smets1995canonical} proposes a canonical decomposition for non-dogmatic mass functions, whose results are denoted as the diffidence function. For an $n$-element frame, a non-dogmatic mass function can be decomposed as $2^n-1$ simple mass functions' fusion and retraction \cite{dubois2020prejudice}. Based on the diffidence function and its dual form, the idempotent combination rules are proposed to fuse non-distinct BPAs \cite{denoeux2008conjunctive}. Pichon \cite{pichon2018canonical} proposes the $t$-canonical decomposition based on the Teugels’ representation of the multivariate Bernoulli distribution. Compared with Smets' result, Pichon's canonical decomposition covers all belief functions and has distinct semantics in multivariate Bernoulli distribution. The aforementioned identical information content representations face the same problem: although the transformations are invertible, their inverse transformations do not guarantee the generation of a valid BPA when the values of the functions are altered. It is the second motivation of this paper.

Isopignistic transformation is also an identical information content transformation but has not been previously discussed. It is derived from the isopignistic domain, which consists of a set of belief functions whose pignistic probability transformations (PPT) yield the same results \cite{dubois2008definition}. The isopignistic transformation involves transferring beliefs within the isopignistic domain. Since the consonant mass function represents the lower bound on the commitment within the isopignistic domain—i.e., the case with the least commitment—its corresponding possibilistic information is developed by Smets as the hyper-cautious transferable belief model (HCTBM). Since the commonality function of the general mass function is not guaranteed to generate a belief function after fusion through triangular norms, HCTBM is only used as an explanation of possibility theory and has not been applied to general. It is the third motivation of this paper.

This paper proposes a novel canonical decomposition developed from isopignistic transformation. Through the isopignistic canonical decomposition, a BPA can be decomposed into two components. The first component, propensity, is a possibility distribution, while the second component, commitment, represents the relationship between the original BPA and the propensity component. Our study not only examines the differences between the proposed method and previous canonical decompositions but also explores the uncertainty composition of the belief function from an isopignistic perspective. Generally, the isopignistic transformation will guide the development of the HCTBM.

The paper is organized as follows: First, some necessary prior concepts in DS theory are introduced. Secondly, isopignistic transformation is proposed based on the belief evolution network and developed into isopignistic canonical decomposition. Finally, we utilize the isopignistic function and ratio to model and handle uncertainty in the HCTBM to demonstrate its advantages.

\section{Basic concepts in Dempster-Shafer theory}

Suppose an uncertain variable $X$, its value is taken in an exhaustible closed set, called frame of discernment, $\Omega=\{\omega_1,\cdots,\omega_n\}$. A restriction $m$, called basic probability assignment (BPA) \cite{shafer1976mathematical} is a mass function to represent the available knowledge about $X$. The mass $m(F_i)$ on the proposition $F_i$ should satisfy $m(F_i)\in [0,1]$ and $\sum_{F_i\subseteq\Omega}m(F_i)=1$. $F_i$ is called a focal set when $m(F_i)>0$, where $i\in \{0,\cdots,2^n-1\}$ denotes the decimal representation of Boolean algebraic binary codes \cite{zhou2023generating}. Thus, the BPA under the $n$-element frame can be represented by a $2^n$-dimensional vector and its order is determined by the binary corresponding to the subset. For the convenient representation, there are some special mass functions. $m$ is a normalized mass function when $m(\emptyset)=0$, which equals a credal set mathematically. $m$ is a Bayesian mass function when the focal sets are singletons, which equals a probability distribution mathematically. $m$ is a simple mass function when there are only two focal sets and one of them is $\Omega$, which can be written as ${F_i}^\sigma\equiv m(F_i)=1-\sigma,~m(\Omega)=\sigma$. $m$ is a non-dogmatic mass function when $m(\Omega)>0$. $m$ is a consonant mass function when the focal sets are nested.

In terms of the set operation-based identical information content representations, belief $Bel$, plausibility $Pl$, implicability $b$ and commonality $q$ functions, are defined as follows
\begin{equation}
	\begin{aligned}
		&Bel(F_i)=\sum_{\emptyset\neq F_j\subseteq F_i}m(F_j),Pl(F_i)=\sum_{F_j\cap F_i\neq \emptyset}m(F_j),\\
		&b(F_i)=\sum_{F_j\subseteq F_i}m(F_j),q(F_i)=\sum_{F_i\subseteq F_j}m(F_j).\\
	\end{aligned}
\end{equation}

The conjunctive combination rule (CCR) and disjunctive combination rule (DCR) can be implemented or retracted efficiently via $q$ and $b$ functions
\begin{equation}\label{ccrdcr}
	\begin{aligned}
		&q_{1\circledtiny{$\cap$}2}(F_i)=q_1(F_i)\times q_2(F_i), q_2(F_i)= \frac{q_{1\circledtiny{$\cap$}2}(F_i)}{q_1(F_i)},\\
		&b_{1\circledtiny{$\cup$}2}(F_i)=b_1(F_i)\times b_2(F_i), b_2(F_i)= \frac{b_{1\circledtiny{$\cup$}2}(F_i)}{b_1(F_i)}.
	\end{aligned}
\end{equation}

In terms of the canonical decomposition-based identical information content representations, the diffidence function \cite{smets1995canonical,dubois2020prejudice} \footnote{Smets' canonical decomposition} $\sigma$ and its dual form $v$ are defined as
\begin{equation}\label{smetscde}
	\begin{aligned}
		\sigma(F_i)=\prod_{F_i\subseteq F_j}q(F_j)^{(-1)^{|F_j|-|F_i|-1}},\\
		v(F_i)=\prod_{F_j\subseteq F_i}b(F_j)^{(-1)^{|F_i|-|F_j|-1}}.
	\end{aligned}
\end{equation}
When the sources are non-distinct, the cautious and bold combination rules are defined as
\begin{equation}\label{caucrbcr}
	\begin{aligned}
		& m_1 \circledsmall{$\wedge$} m_2=\circledlarge{$\cap$}F_i^{\min(\sigma_1(F_i),\sigma_2(F_i))}\\
		& m_1  \circledsmall{$\vee$} m_2=\circledlarge{$\cup$}{F_i}_{\min(v_1(F_i),v_2(F_i))},\\
	\end{aligned}
\end{equation}
where ${F_i}_{v}\equiv m(F_i)=1-v,~m(\emptyset)=v$. 

Another canonical decomposition proposed by Pichon \cite{pichon2018canonical}, called $t$-canonical decomposition is developed from multivariate Bernoulli distribution, the decomposed result, $t$ function is defined as
\begin{equation}\label{pichoncde}
	\begin{aligned}
		t(F_i)=\begin{cases}
			Pl(F_i) & |F_i|=1, \\
			\bigotimes_{i:n\rightarrow 1}\begin{bmatrix}
				1 & 1\\
				Pl(\{\omega_i\}) & 1-Pl(\{\omega_i\})
			\end{bmatrix}\boldsymbol{m} & \text{otherwise}.
		\end{cases}
	\end{aligned}
\end{equation}

This section introduces the basic information representation, processing methods and common identical information content representations in DS theory, which are necessary preliminaries for this paper.

\section{Isopignistic transformation via belief evolution network}

\subsection{Problem description}

When there is no knowledge available to influence an agent's belief, the beliefs should be transferred into singletons for decision making. On the basis of the linearity principle, Smets \cite{smets2005decision} derives the pignistic probability transformation (PPT), which is denoted as
\begin{equation}\label{ppte}
	\begin{aligned}
		BetP_m(\omega)=\sum_{\omega\in F_i}\frac{m(F_i)}{|F_i|}.
	\end{aligned}
\end{equation}
It is evident that the outcome of an unnormalized BPA is still an unnormalized probability distribution.
Due to the irreversible nature of data dimensionality reduction, there exists an infinite number of belief functions with the same $BetP$.
\begin{definition}[Isopignistic domain \& transformation]
	For a probability distribution $\boldsymbol{p}=\{p_1,\cdots ,p_n\}$, its isopignistic domain contains all BPAs whose $BetP$s equal $\boldsymbol{p}$, which is defined as $\mathfrak{Iso}_{\boldsymbol{p}}=\{m|BetP_{m}=\boldsymbol{p}\}.$          
	If a reversible transformation of belief functions, denoted as $m_2=T(m_1)$, satisfies $BetP_{{m}_1}=BetP_{{m}_2}$, it is an isopignistic transformation.
\end{definition}
In this paper, we aim to propose an isopignistic transformation that can transfer a BPA to any position within its corresponding isopignistic domain.

\subsection{Belief evolution network}
Belief evolution network (BEN) is first proposed in \cite{zhou2022belief} and be utilized to develop the probability transformation method, which can be seen as a constrained belief transferring method.
\begin{definition}[Belief evolution network]\label{bend}
	The frame $\Omega$ can be represented as a directed acyclic graph $\mathcal{G}_{\Omega}=\{\mathcal{V},\mathcal{E}\}$, where $\mathcal{V}=\{F_i|\emptyset\neq F_i\subseteq \Omega\}$, $\mathcal{E}=\{(F_i,F_j)||F_i|=|F_j|+1,F_j\subseteq F_i\}$. For nodes where $|F_i| \geq 2$, the coefficient $\tau(F_i) \in [0,1]$ represents the proportion of beliefs transferred from $F_i$. For edges $(F_i, F_j)$, the coefficient $\xi\left(\frac{F_i}{F_j}\right) \in [0,1]$ satisfies $\sum_{F_j \subset F_i} \xi\left(\frac{F_i}{F_j}\right) = 1$, indicating the proportion of beliefs assigned from $F_i$ to $F_j$. The $\mathfrak{B} = \{\mathcal{G}_{\Omega}, {\zeta}, {\xi}\}$ denotes a belief evolution network (BEN). Its schematic diagram is shown in Figure \ref{benf1}. The revision is denoted as $m_\text{B} = T_{\mathfrak{B}}(m)$, and the specific algorithm is shown in Section 2.1 of the supplementary material.
\end{definition}
\begin{figure}[htbp!]
	\centering
	\includegraphics[width=0.6\textwidth]{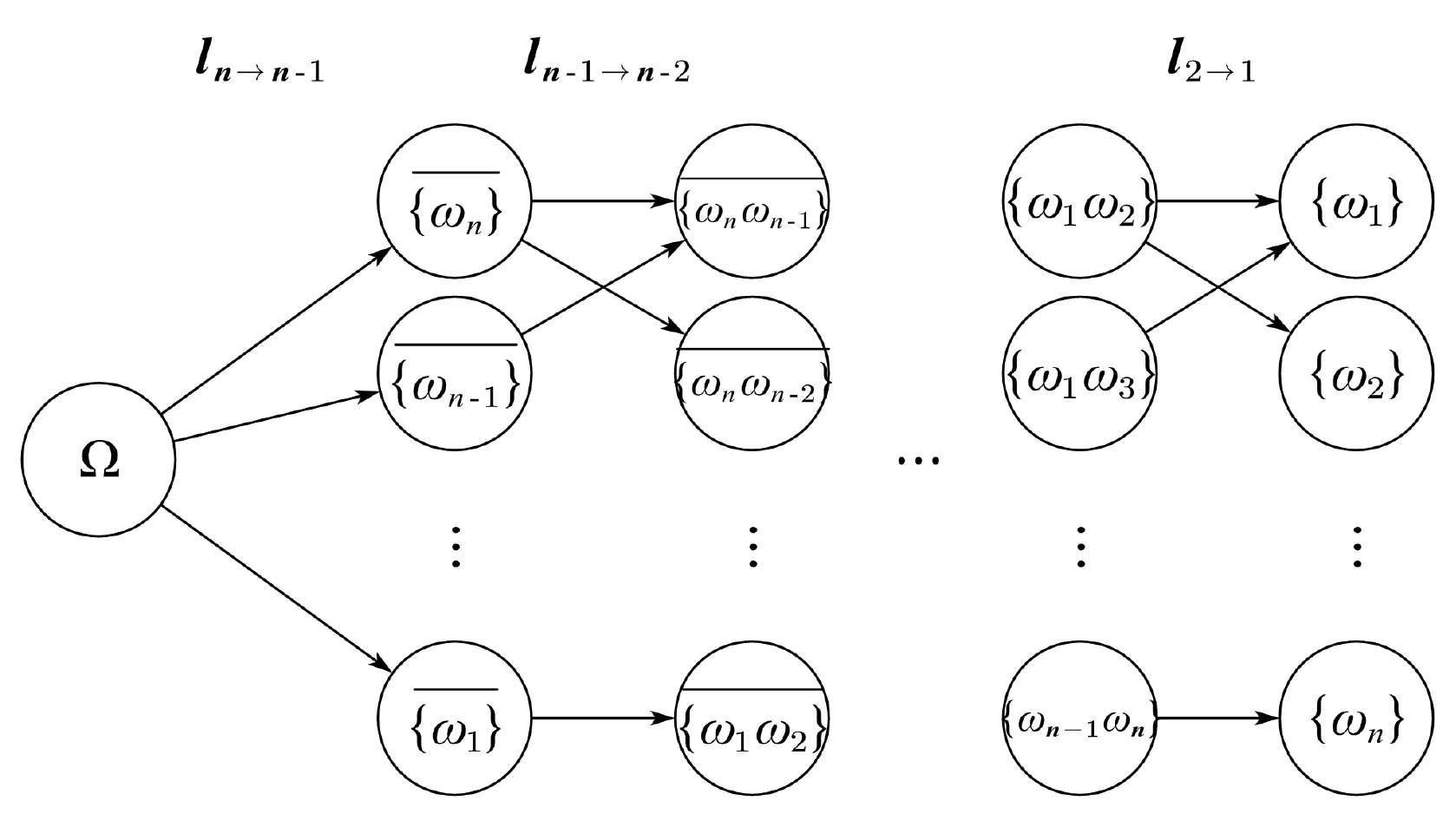}
	\caption{Belief evolution network}
	\label{benf1}
\end{figure}

\begin{proposition}
	If a probability transformation method satisfies upper and lower bounds consistency \footnote{Section V. A in \cite{han2016evaluation}}, i.e., $p(\omega_i) \in [Bel({\omega_i}), Pl({\omega_i})]$, it can be implemented via the BEN.
\end{proposition}

\begin{proof}
	Please refer to Section 1.1 in supplementary material.
\end{proof}

\begin{proposition}\label{ppt_ben}
	If the BEN adheres to the condition $\tau(F_i) = 1$ and $\xi(\frac{F_j}{F_i})=\frac{1}{|F_j|}$, then the outcome of applying the BEN to any BPA will coincide with its $BetP$, i.e., $BetP_{m}=T_{\mathfrak{B}}(m)$.
\end{proposition}

\begin{proof}
	Please refer to Section 1.2 in supplementary material.
\end{proof}

Compared with other belief revision and fusion methods \cite{zhou2017plato, liu2023two}, BEN transfers beliefs between sets of adjacent cardinalities iteratively, achieving a more refined and interpretable transfer process. Therefore, we utilize the BEN to implement a belief revision method that can cover the entire isopignistic domain.

\subsection{Isopignistic transformation}
Building upon Proposition \ref{ppt_ben}, one may conceptualize PPT as a distinctive instance of BEN. When adjusting the values of $\tau$ within the range $[0,1]$, while maintaining uniform distributions for $\xi$, performing the BEN on $m$ can achieve a transient state of the $BetP$. Hence, adjusting $\tau$ can implement an isopignistic transformation. However, since edges direct from propositions with larger cardinalities to those with smaller cardinalities and $\tau$ is a positive value, this transformation can not cover the entire isopignistic domain. To address this issue, we perform the BEN twice for each transformation, including forward and backward transfers, to ensure the revision can reach any positions of isopignistic domain.

\begin{definition}[Isopignistic transformation]
	Consider two BPAs $m_1$ and $m_2$ under the frame $\Omega$, coexisting within an isopignistic domain. The transformation $m_2=T_{\mathfrak{IB}}(m_1)$ is called the isopignistic transformation. $\mathfrak{IB}$ has two forms $\{\mathcal{G}_{\Omega},\tau\}$ and $\{\mathcal{G}_{\Omega},\zeta\}$, where $\zeta$ is isotransformation function, which represents the transferred beliefs and ${\tau}$ is isotransformation ratio, which represents its proportion. Under an isopignistic transformation, they have identical information content. The generations of ${\tau}$ and ${\zeta}$ can be written as ${\tau}=IT^{\tau}(m_1,m_2)$ and ${\zeta}=IT^{\zeta}(m_1,m_2)$, and the specific algorithm is shown in Section 2.2 in supplementary material.
\end{definition}

% \begin{algorithm}[tb]
	% \caption{Isopignistic transformation: \\${\tau}=IT^{\tau}(m_1,m_2)$, ${\zeta}=IT^{\zeta}(m_1,m_2)$}
	% \label{iso_rcd}
	% \textbf{Input}: Original BPA $m_1$, Outcome BPA $m_2$\\
	% % \textbf{Parameter}: Optional list of parameters\\
	% \textbf{Output}: Isotransformation function $\zeta$, Isotransformation ratio $\tau$
	% \begin{algorithmic}[1] %[1] enables line numbers
		% \STATE \textit{\%Forward transfers:}
		% \FOR{$t\leftarrow 1$ to $|\Omega|-1$}
		% \FOR{$F_i\subseteq \Omega;|F_i|=|\Omega|-t+1$}
		% \STATE $\zeta(F_i)\leftarrow m_1(F_i)-m_2(F_i)$
		% \IF{$\zeta(F_i)>0$}
		% \STATE $\tau(F_i)\leftarrow \frac{\zeta(F_i)}{m_1(F_i)}$
		% \STATE $m_1(F_i)\leftarrow m_2(F_i)$
		% \FOR{$F_j\subset F_i$; $|F_j|+1=|F_i|$}
		% \STATE $m_1(F_j)\leftarrow m_1(F_j)+\frac{\zeta(F_i)}{|F_i|}$
		% \ENDFOR
		% \ENDIF
		% \ENDFOR
		% \ENDFOR
		% \STATE \textit{\% Backward transfers:}
		% \FOR{$t\leftarrow 1$ to $|\Omega|-1$}
		% \FOR{$i\leftarrow 0 to 2^{|\Omega|}-1$}
		% \IF{$|F_i|=t+1$ and $\zeta(F_i)<0$}
		% \STATE $\tau(F_i) \leftarrow \frac{\zeta(F_i)}{|F_i|\min_{|F_j|=|F_i|-1,F_j\subset F_i}m_1(F_j)}$
		% \STATE $m_1(F_i)\leftarrow m_2(F_i)$
		% \FOR{$F_j\subset F_i$ and $|F_j|=|F_i|-1$}
		% \STATE $m_1(F_j)\leftarrow m_1(F_j)+\frac{\zeta(F_i)}{|F_i|}$
		% \ENDFOR
		% \ENDIF
		% \ENDFOR
		% \ENDFOR
		% \STATE \textbf{return} $\zeta$, $\tau$
		% \end{algorithmic}
	% \end{algorithm}

There are the following two differences between the isopignistic transformation and the BEN. Firstly, the isopignistic transformation constrains that $\xi(\frac{F_i}{F_j}) = \frac{1}{|F_i|}$, which preserves the pignistic consistency. Secondly, backward transfer is introduced to adjust the commitment degrees among the focal sets with the same cardinality, which extends the range of values for $\tau$ to $[-1,1]$.

\begin{example}\label{e1}
	Please refer to Section 3.1 in supplementary material.
\end{example}

The isotransformation function and ratio both produce the same outcomes, yet they play distinct roles in information processing. We utilize the following propositions to detail their differences.

\begin{proposition}[Role of ${\zeta}$]\label{role_zeta}
	Consider an isopignistic transformation $m_2 = T_{\mathfrak{IB}}(m_1)$. If its isotransformation function is $\zeta$, then $\zeta'$, satisfying $\zeta'(F_i) = -\zeta(F_i)$, will serve as the isotransformation function for the inverse transformation $m_1 = T_{\mathfrak{IB}}'(m_2)$.
\end{proposition}

\begin{proof}
	Please refer to Section 1.3 in supplementary material.
\end{proof}

\begin{proposition}[Role of ${\tau}$]\label{role_tau}
	If all values of $\tau$ lie within the range $[-1,1]$, for any BPA $m$, the outcome of the transformation $m'$ using $\mathfrak{IB}=\{\mathcal{G}_{\Omega},\tau\}$ remains a BPA. In other words, it adheres to the conditions $m'(F_i)\in[0,1]$ and $\sum_{F_i\subseteq \Omega}m'(F_i)=1$.
\end{proposition}

\begin{proof}
	Please refer to Section 1.4 in supplementary material.
\end{proof}

\begin{proposition}[Transmittability]\label{trans_iso}
	Consider two isopignistic transformations, $\mathfrak{IB}_1=\{\mathcal{G}_{\Omega},{\zeta}_1\}$ and $\mathfrak{IB}_2=\{\mathcal{G}_{\Omega},{\zeta}_2\}$, they satisfy $m_2=T_{\mathfrak{IB}_1}(m_1)$ and $m_3=T_{\mathfrak{IB}_2}(m_2)$. The transformation from $m_1$ to $m_3$ can be written as $m_3=T_{\mathfrak{IB}_{1+2}}(m_1)$, where $\mathfrak{IB}_{1+2}=\{\mathcal{G}_{\Omega},{\zeta}_{1}+{\zeta}_{2}\}$.
\end{proposition}

\begin{proof}
	Please refer to Section 1.5 in supplementary material.
\end{proof}

\begin{proposition}[Ergodicity]\label{p6}
	Consider a $m$ under the frame $\Omega$, for arbitrary BPA $m'$ satisfying $m'\in \mathfrak{Iso}_{BetP_{m}}$, there must exist an isopignistic transformation with $\mathfrak{IB}=\{\mathcal{G}_{\Omega},\zeta\}$ satisfying $m'=T_{\mathfrak{IB}}(m)$.
\end{proposition}

\begin{proof}
	Please refer to Section 1.6 in supplementary material.
\end{proof}

\subsection{Discussion}
In this section, we introduce a novel belief revision method designed to transform BPAs within the isopignistic domain. This method is implemented through a modified BEN and has two forms: the isotransformation function and the ratio. On the one hand, the isotransformation function specifies the beliefs involved in the revision process, allowing for both inverse and sequential transformations. On the other hand, the isotransformation ratio represents the proportion of beliefs being transferred, and can be used to perform transformations even when the target BPA is unknown.

\section{Isopignistic canonical decomposition}

\subsection{Motivation}
In the introduction section, we discuss the motivations of this paper from three perspectives: uncertainty representation, canonical decomposition and HCTBM, respectively. In conjunction with the isopignistic transformation proposed in the previous section, we will explore why these motivations can be addressed through isopignistic canonical decomposition.

\subsubsection{Uncertainty representation} Both possibility and probability distributions can be modeled using mass functions, where the former corresponds to consonant mass functions and the latter corresponds to Bayesian mass functions. Hence, if they can be situated within an isopignistic domain and connected through the BEN, the isotransformation function and ratio will provide a unified relationship between them and the general mass function.

\subsubsection{Canonical decomposition} There are two well-known methods from previous research. Smets' decomposition (Eq. (\ref{smetscde})) does not encompass all mass functions, while Pichon's decomposition (Eq. (\ref{pichoncde})) lacks a distinct interpretation within the belief function framework. Additionally, neither method allows for the reconstruction of BPAs from their elementary components. Proposition \ref{role_tau} demonstrates that the isotransformation ratio offers a method to explore unknown mass functions within the isopignistic domain. Hence, if a canonical decomposition is developed from the isotransformation ratio, the reconstruction can be achieved through the inverse canonical decomposition.

\subsubsection{Hyper-cautious TBM} Its information fusion is achieved by applying the minimum t-norm to the commonality functions. However, only when the original BPAs are consonant mass functions can the fused commonality function be inverted to a BPA. For a general BPA, if first transform it into a consonant mass function, then fuse it using the minimum t-norm, and finally invert it back to the general form, its information fusion under the HCTBM will be achieved. Since the isotransformation ratio provides a pathway for reconstructing BPAs, the issue can be addressed from the perspective of isopignistic transformation.

\subsection{Isopignistic canonical decomposition}

Since a consonant mass function uniquely corresponds to a possibility distribution, where the possibility measure represents the propensity of the element. In addition, both the isotransformation function and ratio are mappings: $\{F_i \mid |F_i| > 1, F_i \subseteq \Omega\} \rightarrow [-1,1]$, which can be used to describe the path from a possibility distribution to the target BPA. Hence, we decompose a BPA into two components: propensity and commitment.

\begin{definition}[Isopignistic canonical decomposition]
	Given a BPA $m$ under an $n$-element frame $\Omega$, it can be decomposed as a propensity component and a commitment component. Propensity is an $n$-dimensional possibility distribution $Poss$
	\begin{equation}
		\begin{aligned}
			Poss(\omega_j)=\sum_{\omega_i\in\Omega}\min(BetP_{m}(\omega_i),BetP_{m}(\omega_j)),
		\end{aligned}
	\end{equation}
	where $BetP_m$ can be implemented via Eq. (\ref{ppte}). Commitment is similar with the isopignistic transformation, which has two forms, isotransformation function and ratio.
	\begin{equation}
		\begin{aligned}
			\tau = IT^{\tau}(m_{\rm{c}},m), \zeta = IT^{\zeta}(m_{\rm{c}},m),
		\end{aligned}
	\end{equation}
	where $$m_{\rm{c}}=Poss(\omega(t))-Poss(\omega(t+1)),$$ and $\omega(t)$ is the element with the $t$th largest possibility measure. Specially, $Poss(\omega(n+1))=0$. The above elementary pieces can be represented as two forms of information granule, denoted as the isopignistic function $Iso^{\tau}$ and $Iso^{\zeta}$,
	\begin{equation}\label{isocde}
		\begin{aligned}
			Iso^{\zeta}(F_i)=\begin{cases}
				m(\emptyset) & F_i = \emptyset, \\
				Poss(\omega_j) & F_i = \{\omega_j\},\\
				\zeta(F_i) & |F_i|>1,\\
			\end{cases}\\
			Iso^{\tau}(F_i)=\begin{cases}
				m(\emptyset) & F_i = \emptyset, \\
				Poss(\omega_j) & F_i = \{\omega_j\},\\
				\tau(F_i) & |F_i|>1.
			\end{cases}
		\end{aligned}
	\end{equation}
	Hence, the isopignisitc canonical decomposition is denoted as $Iso^\zeta = IsoCD[\zeta](m)$ and $Iso^\tau = IsoCD[\tau](m)$.
\end{definition}

\begin{example}\label{e2}
	Please refer to Section 3.2 in supplementary material.
\end{example}

As an identical information content transformation, its inverse process is denoted as $m = IsoCD^{-1}[\tau](Iso^{\tau})$ or $ m = IsoCD^{-1}[\zeta](Iso^{\zeta})$, whose specific Algorithms are shown in the Sections 2.3 and 2.4 in supplementary material. The isopignistic canonical decomposition is proposed based on the isopignistic transformation, which decomposes the target BPA into a possibility distribution and an isotransformation function (or ratio) by establishing an invertible transformation between the consonant mass function and the target BPA within the isopignistic domain. Thus, all the properties of isopignistic transformation discussed previously also apply to the corresponding canonical decomposition.

\subsection{Properties}

To demonstrate the rationality of the isopignistic canonical decomposition, its properties are further examined within the belief function framework.

\begin{proposition}[Lower and upper bounds]\label{p7}
	Consider a BPA $m$ under an $n$-element frame $\Omega$. The commitment component resulting from the isopignistic canonical decomposition has lower and upper bounds, which correspond to probability and possibility distribution, respectively. In terms of the isopignistic ratio, its upper bound satisfies $\forall |F_i|>1$, $Iso^{\tau}(F_i)=1$, and its lower bound satisfies $\forall |F_i|>1$, $Iso^{\tau}(F_i)\leq 0$. In terms of the isopignistic function, its upper bound satisfies $\forall |F_i|>1$, $Iso^{\tau}=IT^{\tau}(m_{\rm{c}}, BetP_m)$, and its lower bound satisfies  $\forall |F_i|>1$, $Iso^{\zeta}(F_i)=0$.
\end{proposition}

\begin{proof}
	Please refer to Section 1.7 in supplementary material.
\end{proof}

\begin{proposition}[Reversible construction]\label{p8}
	Consider a power set $2^\Omega$ of a frame $\Omega$, and a mapping $pc: 2^\Omega \rightarrow [-1, 1]$. If $pc$ satisfies the following requirements: $\forall |F_i| = 1$, $pc(F_i) \in [0,1]$; $\forall |F_i| > 1$, $pc(F_i) \in [-1,1]$; and $pc(\emptyset) = 1 - \max_{\omega \in \Omega} pc(\{\omega\})$, then it can be reconstructed as a BPA via $m = IsoCD^{-1}[\tau](pc)$.
\end{proposition}

\begin{proof}
	Please refer to Section 1.8 in supplementary material.
\end{proof}

\begin{remark}\label{r1}
	During the reconstruction process of BPA, in addition to the ratio itself, the possibility measure will also affect the transferred beliefs. Consider a subset $F_i$, if $m_\text{c}(F_i)=0$, no matter the value of the ratio, there is no beliefs will be transferred from $F_i$. Hence, there may be infinite number of $pc$s point to a same $m$, but only when $pc$ satisfies $pc=IsoCD^{-1}[\tau](IsoCD[\tau](pc))$, it will be denoted as isopignistic ratio.
\end{remark}

\begin{remark}\label{r3}
	Proposition \ref{p6} indicates that for a BPA $m$, the isotransformation function can implement an isopignistic transformation at any position within $\mathfrak{Iso}_{BetP_m}$. Therefore, any BPA can be decomposed into a unique isopignistic function and ratio. When using the ratio to reconstruct a BPA, and with the propensity component determined, any position within the corresponding isopignistic domain can be implemented through the ratio.
\end{remark}

\begin{example}\label{e3}
	Please refer to Section 3.3 in supplementary material.
\end{example}

Example \ref{e3} utilizes a typical instance to illustrate Remark \ref{r1}. When the propensity component is fixed, different ratios may correspond to the same BPA, but the ratios corresponding to different BPAs must be distinct. Therefore, this does not affect the reconstructed BPA characteristics of the isopignistic ratio.

\section{Advantages}

In this section, the advantages of the isopignistic canonical decomposition will be discussed in relation to the three previous motivations: uncertainty representation, canonical decomposition, and HCTBM.
\subsection{Uncertainty representation}
The transformation approaches between possibilistic and probabilistic information have been discussed from various perspective \cite{dubois2004probability,zadeh2014note}. Each pair of bijective transformations can be positioned as a Bayesian mass function and a consonant mass function under the same belief structure. However, there has been no previous research on how the transformation is implemented. In terms of the Smets' canonical decomposition, since the Bayesian mass function is dogmatic, it can not be decomposed via diffidence function. In terms of Pichon's canonical decomposition, when the contour function $Pl({\omega})$ is fixed and satisfies $\sum_{\omega \in \Omega} Pl({\omega}) \leq 1$, there exists a pair of $t$ functions to represent the possibility and probability distributions, respectively. However, when $\sum_{\omega \in \Omega} Pl({\omega}) \leq 1$, adjusting the $t$ functions to maintain the additivity of the plausibility function will result in a mass function that is not a valid BPA. In terms of isopignistic canonical decomposition, as shown in Proposition \ref{p7}, the upper and lower bound of the isopignistic ratio correspond the unique probability and possibility distribution, respectively. In addition to providing a pathway between probability and possibility distributions, any mass function can be interpreted as a transient state of their transformation using the proposed method. As shown in Figure \ref{aaai_fig}, different adjustments of the ratio lead to different paths.

\begin{figure}[htbp!]
	\centering
	\includegraphics[width=0.6\textwidth]{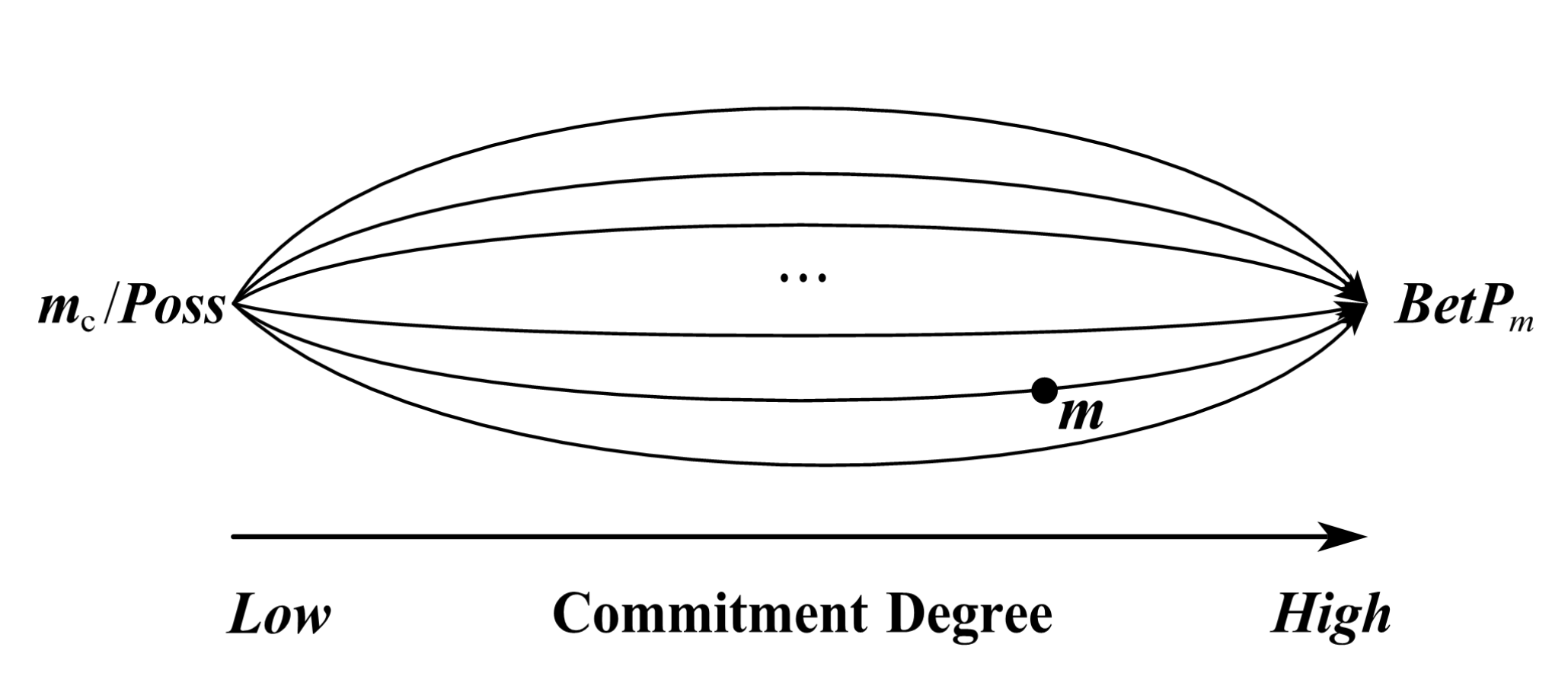}
	\caption{A general BPA can be interpreted as a transient state from possibility to probability}
	\label{aaai_fig}
\end{figure}

\subsection{Canonical decomposition}

Since isopignistic ratio can be used to reconstruct a BPA. Example \ref{e_ad1} compares the reconstructing abilities of these canonical decompositions. When the function values of multi-element subsets vary within the corresponding value ranges, only the reconstruction of the isopignistic ratio ensures that the beliefs of mass function are within the range $[0,1]$.

\begin{example}\label{e_ad1}
	Please refer to Section 3.4 in supplementary material.
\end{example}

In addition to mathematically proving the feasibility of reconstructing a BPA, it is equally important to discuss the significance of the propensity component and commitment component. In the realm of the propensity component, since it can be viewed as a possibility distribution, it reflects the degree of non-negation for elements being the true value of the uncertain variable. When it is unnormalized, i.e., $\max_{\omega\in\Omega} Poss(\omega) < 1$, it also quantifies the importance of the body of evidence. In the realm of the commitment component, when the $Iso^{\tau}\in(0,1]$, it quantifies the negation degree of the elements contained in subsets. For a subset $F_i$ and $Iso^{\tau}(F_i)=1$, it means that the beliefs of $F_i$ is uniformly assigned to $F_j\in\{F_j|F_j = F_i\setminus \{\omega\}, \omega\in F_i\}$. Since the beliefs are transferred from subsets with larger cardinalities to those with smaller cardinalities, this process results in a BPA with a higher degree of commitment. When $Iso^{\tau} \in [-1,0)$, it indicates that beliefs are being transferred to subsets with larger cardinalities. 

% \begin{remark}
	%     $Iso^{\tau} \in [-1,0)$ does not imply a lower degree of commitment; rather, it suggests that the commitment degree flows among subsets with the same cardinality. This is because when a negative ratio can be expressed as a negative isopignistic function, it implies that all child nodes of $F_i$ have positive beliefs. However, the initial mass function is consonant, meaning it has only one focal set at each cardinality.Therefore, in the previous forward transfers, other nodes in the same layer as $F_i$ must have transferred beliefs to nodes with a cardinality of $|F_i|-1$, which were then backward transferred to $F_i$. Since the ratio represents the proportion of the minimum beliefs transferred among all child nodes, the beliefs transferred backward must be less than those transferred forward, which does not result in a lower commitment degree.
	% \end{remark}

Therefore, the reconstruction of a BPA can be achieved using the ratio and possibility distribution, which is the most significant advantage of isopignistic canonical decomposition compared to other methods.

\subsection{HCTBM}

In HCTBM \cite{smets2000theorie}, the propensity component, contour functions, and commonality functions are equivalent for the consonant mass function, and the minimum t-norm operator from possibility theory is used to combine the bodies of evidence. This paper extends the above idea to the general BPAs via isopignistic canonical decomposition. 

\subsubsection{Uncertainty measures}
Since the isopignistic function represents the specific beliefs in the transformation process, it is utilized to develop the novel uncertainty measures from the HCTBM perspective. Specificity, as the most significant uncertainty metric in possibility theory, has been extended in belief function framework \cite{yager2008entropy}. Given a BPA $m$ under the frame $\Omega$, its specificity measure is defined as
\begin{equation}
	\begin{aligned}
		S(m)=\sum_{F_i\subseteq \Omega}\frac{m(F_i)}{|F_i|}.
	\end{aligned}
\end{equation}
When the $m$ is a consonant mass function, and its corresponding possibility distribution is $\pi$, it will equal the specificity measure of the possibility distribution \cite{yager1981measurement}
\begin{equation}\label{yspe}
	\begin{aligned}
		S(\pi)=\int_{0}^{\alpha_{\max}}\frac{1}{|\pi_{\alpha}|}d\alpha,
	\end{aligned}
\end{equation}
where $\alpha$s are the possibilities of elements and $|\pi_{\alpha}|$ is the cardinality of the crisp set $\pi_\alpha=\{\omega_i|\pi(\omega_i)\geq \alpha\}$.

In possibility theory, specificity is utilized to measure the degree of an information granule describing a variable $X$ points to only one element. When the specificity reaches its maximum value of $1$, it indicates that the possibility of an element $\omega_i$ is $1$, while the possibilities of the other elements are $0$, and the value of $X$ is determined as $\omega_i$. In DS theory, $S(m)$ cannot play the same role as the possibility theory anymore. For the Bayesian mass function, its beliefs are focused on the singletons, whose specificity measure reaches the maximum value of $1$. The $X$ degrade to a random variable, and its value still can not be determined. To address this issue, the specificity measure is divided into two sub-metrics based on the propensity and commitment components.

\begin{definition}[specificity measures based on the isopignistic function]\label{smd}
	Given a normalized BPA $m$ under the frame $\Omega$, and its isopignistic function is $Iso^{\zeta}$. Capture the singletons of $Iso^{\zeta}$ as a possibility distribution $Poss_{m}$, and the \textbf{propensity specificity} measure is defined as
	\begin{equation}\label{spe}
		\begin{aligned}
			S_{\rm{p}}(m)=S(Poss_m),
		\end{aligned}
	\end{equation}
	where $S(Poss_m)$ is implemented via Eq.(\ref{yspe}). The \textbf{commitment specificity} measure is defined as
	\begin{equation}\label{sce}
		\begin{aligned}
			S_{\rm{c}}(m)=\frac{\sum_{F_i\subseteq \Omega; |F_i|\geq2}Iso^{\zeta}(F_i)}{\sum_{F_i\subseteq \Omega; |F_i|\geq2}(|F_i|-1)\times m_{\rm{c}}(F_i)},
		\end{aligned}
	\end{equation}
	where $m_{\rm{c}}$ is the consonant function generated from the $Poss_m$.
\end{definition}

\begin{proposition}[Range of propensity specificity]
	Given a BPA $m$ under an $n$-element frame, its range of propensity specificity is $[\frac{1}{n},1]$. When subsets with equal cardinalities hold the same beliefs, the propensity specificity reaches the minimum value $\frac{1}{n}$. When $m$ satisfies $m(\{\omega\})=1$, the propensity specificity reaches the maximum value $1$.
\end{proposition}

\begin{proof}
	Please refer to Section 1.9 in supplementary material.
\end{proof}

\begin{proposition}[Range of commitment specificity]\label{p10}
	Given a BPA $m$ under an $n$-element frame, its range of propensity specificity is $[0,1]$. When $m$ is consonant, the commitment specificity attains the minimum value $0$. When $m$ is Bayesian, the commitment specificity reaches the maximum value $1$. Notably, in the scenario where $m$ represents a deterministic event, commitment specificity becomes meaningless.
\end{proposition}

\begin{proof}
	Please refer to Section 1.10 in supplementary material.
\end{proof}

The introduction of propensity specificity and commitment specificity aims to extend the original goal of specificity. Within the framework of DS theory, find a measure for representing the degree of an information granule points to a specific element. Example \ref{e5} analysis their advantages via a numerical example.

\begin{example}\label{e5}
	Please refer to Section 3.5 in supplementary material.
\end{example}

\subsubsection{Information fusion}

Utilizing the reversible construction of isopignistic ratio, we extend the hyper-cautious combination rules to general BPAs.
\begin{definition}[Combination rules in hyper-cautious TBM]
	Given two BPAs $m_1$ and $m_2$ under the frame $\Omega$, the hyper-cautious combination rules of them are defined as
	\begin{equation}\label{hctbmcr}
		\begin{aligned}
			&Iso^{\tau}_{m_1}=IsoD[\tau](m_1),~Iso^{\tau}_{m_2}=IsoD[\tau](m_2)\\
			& Iso^{\tau}_{m_1\circledtiny{$\rm{H}$}m_2}(F_i)=\begin{cases}
				Iso^{\tau}_{m_1}(F_i) {\rm{H}} Iso^{\tau}_{m_2}(F_i) & |F_i|=1\\
				\frac{Iso^{\tau}_{m_1}(F_i)+Iso^{\tau}_{m_2}(F_i)}{2} & |F_i|\geq2,\\
				1-K & F_i=\emptyset\\
			\end{cases}\\
			& m_1\circledsmall{$\rm{H}$}m_2=IsoD^{-1}[\tau] (Iso^{\tau}_{m_1\circledtiny{$\rm{H}$}m_2} ),
		\end{aligned}
	\end{equation}
	where $$K=1-\max_{\omega_i\in \Omega} (Iso^{\tau}_{m_1}(\{\omega_i\})\circledsmall{$\rm{H}$} Iso^{\tau}_{m_2}(\{\omega_i\})),$$ and $\rm{H}=\{\top_{\boldsymbol{m}},\top_{\boldsymbol{p}},\bot_{\boldsymbol{m}},\bot_{\boldsymbol{p}}\}$, where $\top_{\boldsymbol{m}}(a,b)=\min(a,b)$, $\top_{\boldsymbol{p}}(a,b)=a\cdot b$, $\bot_{\boldsymbol{m}}(a,b)=\max(a,b)$, $ \bot_{\boldsymbol{p}}(a,b)=a+b-a\cdot b$.
\end{definition}

This paper focuses on highlighting the feasibility of utilizing an isopignistic function to handle uncertainty in hyper-cautious TBM. Therefore only discuss minimum and product t-norms in the conjunctive case and maximum and probabilistic t-conorms in the disjunctive case, and other operators containing parameters and their effects in specific applications will be analyzed in further research. In terms of reliability and dependence in the source state, the proposed combination rules based on the above $4$ operators under the hyper-cautious TBM perspective are equivalent to CauCR, CCR, DCR, and BCR (Eqs. (\ref{ccrdcr}) and (\ref{caucrbcr})), respectively. Example \ref{ecr1} utilizes a toy numerical experiment to show the difference of combination rules between hyper-cautious TBM and TBM.

\begin{example}\label{ecr1}
	Please refer to Section 3.6 in supplementary material.
\end{example}

To further show how they differ from the TBM combination rules, with respect to their validity, we will discuss their properties based on the requirements in \cite{abellan2021combination}.
\begin{proposition}[commutativity]\label{p11}
	Exchanging the two BPAs will not affect the outcome of the fusion $m_1\circledsmall{$\rm{H}$}m_2$=$m_2\circledsmall{$\rm{H}$}m_1$.
\end{proposition}

\begin{proof}
	Please refer to Section 1.11 in supplementary material.
\end{proof}

\begin{proposition}[quasi-associativity]\label{p12}
	It does not satisfy the associativity, i.e., $m_1\circledsmall{$\rm{H}$}m_2\circledsmall{$\rm{H}$}m_3\neq m_1\circledsmall{$\rm{H}$}(m_2\circledsmall{$\rm{H}$}m_3)$. However, it holds quasi-associativity, i.e., it can provide a $k$ sources information fusion frame.
\end{proposition}

\begin{proof}
	Please refer to Section 1.12 in supplementary material.
\end{proof}

\begin{proposition}[idempotency]\label{p13}
	For the minimum rule and maximum rule, they satisfy $m\circledsmall{$\top$}_{\boldsymbol{m}}/\circledsmall{$\bot$}_{\boldsymbol{m}}m=m$.
\end{proposition}

\begin{proof}
	Please refer to Section 1.13 in supplementary material.
\end{proof}

\begin{proposition}[quasi-neural element]\label{p14}
	For the product rule and probabilistic rule, their quasi-neural elements are $m_\Omega$ and $m_\emptyset$, respectively, i.e., for $m\circledsmall{$\top$}_{\boldsymbol{p}}m_\Omega/ \circledsmall{$\bot$}_{\boldsymbol{p}}m_\emptyset=m'$, it satisfies $BetP_{m}=BetP_{m'}$, and $S_\textrm{c}(m')\leq S_\textrm{c}(m)$. Especially, when $m$ is consonant, it satisfies $m\circledsmall{$\top$}_{\boldsymbol{p}}m_\Omega/ \circledsmall{$\bot$}_{\boldsymbol{p}}m_\emptyset=m$.
\end{proposition}

\begin{proof}
	Please refer to Section 1.14 in supplementary material.
\end{proof}

\begin{proposition}[informative monotonicity]\label{p15}
	For the combination $m_1$\circledsmall{$\top$}$m_2$=$m$, their propensity components satisfy
	$Iso_{m_j}(\{\omega_i\})\geq Iso_{m}(\{\omega_i\})$, where $j=\{1,2\}$ and $\omega_i\in\Omega$. For the combination $m_1$\circledsmall{$\bot$}$m_2$=$m$, their propensity components satisfy $Iso_{m_j}(\{\omega_i\})\leq Iso_{m}(\{\omega_i\})$, where $j=\{1,2\}$ and $\omega_i\in\Omega$.
\end{proposition}

\begin{proof}
	Please refer to Section 1.15 in supplementary material.
\end{proof}

Thus, isopignistic canonical decomposition provides a new perspective on the measurement and fusion of evidential information and enriches the information processing framework of HCTBM.

\section{Conclusions}
The main contributions of this paper are as follows: First, the isopignistic transformation is introduced for the first time as a belief revision method that does not alter decision probabilities. Second, an isopignistic canonical decomposition is developed using the proposed transformation, enabling an interpretable reconstruction of BPAs. Finally, the uncertainty measurements and combination rules in HCTBM are extended based on the proposed method.

In future research, we will further analyze the implications of the isopignistic function and its role in other belief revision methods, as well as develop the pathway from data and knowledge to BPA through the isopignistic ratio. More generally, the information processing methods in HCTBM will be explored in greater depth, with the goal of achieving unified uncertainty management for both probabilistic and possibilistic information.

\section*{Acknowledgment}

The work is partially supported by the National Natural Science Foundation of China (Grant No. 62373078).

\bibliography{mybibfile}

\begin{thebibliography}{10}

\bibitem{abellan2021combination}
Joaqu{\'\i}n Abell{\'a}n, Seraf{\'\i}n Moral-Garc{\'\i}a, and Mar{\'\i}a~D
  Ben{\'\i}tez.
\newblock Combination in the theory of evidence via a new measurement of the
  conflict between evidences.
\newblock {\em Expert Systems with Applications}, 178:114987, 2021.

\bibitem{dempster2008upper}
Arthur~P Dempster.
\newblock Upper and lower probabilities induced by a multivalued mapping.
\newblock In {\em Classic works of the Dempster-Shafer theory of belief
  functions}, pages 57--72. Springer, 2008.

\bibitem{denoeux2008conjunctive}
Thierry Den{\oe}ux.
\newblock Conjunctive and disjunctive combination of belief functions induced
  by nondistinct bodies of evidence.
\newblock {\em Artificial Intelligence}, 172(2-3):234--264, 2008.

\bibitem{dubois2020prejudice}
Didier Dubois, Francis Faux, and Henri Prade.
\newblock Prejudice in uncertain information merging: Pushing the fusion
  paradigm of evidence theory further.
\newblock {\em International Journal of Approximate Reasoning}, 121:1--22,
  2020.

\bibitem{dubois2004probability}
Didier Dubois, Laurent Foulloy, Gilles Mauris, and Henri Prade.
\newblock Probability-possibility transformations, triangular fuzzy sets, and
  probabilistic inequalities.
\newblock {\em Reliable computing}, 10(4):273--297, 2004.

\bibitem{dubois2001new}
Didier Dubois, Henri Prade, and Philippe Smets.
\newblock New semantics for quantitative possibility theory.
\newblock In {\em Symbolic and Quantitative Approaches to Reasoning with
  Uncertainty: 6th European Conference, ECSQARU 2001 Toulouse, France,
  September 19--21, 2001 Proceedings 6}, pages 410--421. Springer, 2001.

\bibitem{dubois2008definition}
Didier Dubois, Henri Prade, and Philippe Smets.
\newblock A definition of subjective possibility.
\newblock {\em International journal of approximate reasoning}, 48(2):352--364,
  2008.

\bibitem{han2016evaluation}
Deqiang Han, Jean Dezert, and Zhansheng Duan.
\newblock Evaluation of probability transformations of belief functions for
  decision making.
\newblock {\em IEEE Transactions on Systems, Man, and Cybernetics: Systems},
  46(1):93--108, 2016.

\bibitem{liu2023two}
Likang Liu, Keke Sun, Chunlai Zhou, and Yuan Feng.
\newblock Two views of constrained differential privacy: Belief revision and
  update.
\newblock In {\em Proceedings of the AAAI Conference on Artificial
  Intelligence}, pages 6450--6457, 2023.

\bibitem{pichon2018canonical}
Fr{\'e}d{\'e}ric Pichon.
\newblock Canonical decomposition of belief functions based on teugels’
  representation of the multivariate bernoulli distribution.
\newblock {\em Information Sciences}, 428:76--104, 2018.

\bibitem{shafer1976mathematical}
Glenn Shafer.
\newblock {\em A mathematical theory of evidence}, volume~42.
\newblock Princeton university press, 1976.

\bibitem{smets1995canonical}
Philippe Smets.
\newblock The canonical decomposition of a weighted belief.
\newblock In {\em IJCAI}, volume~95, pages 1896--1901, 1995.

\bibitem{smets2000theorie}
Philippe Smets.
\newblock La th{\'e}orie des possibilit{\'e}s quantitatives
  {\'e}pist{\'e}miques vue comme un modele de croyances transf{\'e}rables tres
  prudent. quantified epistemic possibilty theory seens as an hyper cautious
  transferable belief model.
\newblock {\em LFA La Rochelle}, pages 343--353, 2000.

\bibitem{smets2005decision}
Philippe Smets.
\newblock Decision making in the tbm: the necessity of the pignistic
  transformation.
\newblock {\em International journal of approximate reasoning}, 38(2):133--147,
  2005.

\bibitem{solaiman2019possibility}
Basel Solaiman and {\'E}loi Boss{\'e}.
\newblock {\em Possibility Theory for the Design of Information Fusion
  Systems}.
\newblock Springer, 2019.

\bibitem{yager1981measurement}
Ronald~R Yager.
\newblock Measurement of properties on fuzzy sets and possibility
  distributions.
\newblock In {\em Proc. 3rd Intern. Seminar on Fuzzy Set Theory. Johannes
  Kepler-University, Linz}, pages 211--222, 1981.

\bibitem{yager2008entropy}
Ronald~R Yager.
\newblock Entropy and specificity in a mathematical theory of evidence.
\newblock {\em Classic works of the Dempster-Shafer theory of belief
  functions}, pages 291--310, 2008.

\bibitem{zadeh2014note}
Lotfi~A. Zadeh.
\newblock A note on similarity-based definitions of possibility and
  probability.
\newblock {\em Information Sciences}, 267:334--336, 2014.

\bibitem{zhou2017plato}
Chunlai Zhou, Biao Qin, and Xiaoyong Du.
\newblock Plato's cave in the dempster-shafer land-the link between pignistic
  and plausibility transformations.
\newblock In {\em IJCAI}, pages 4676--4682, 2017.

\bibitem{zhou2023generating}
Qianli Zhou and Yong Deng.
\newblock Generating sierpinski gasket from matrix calculus in dempster--shafer
  theory.
\newblock {\em Chaos, Solitons \& Fractals}, 166:112962, 2023.

\bibitem{zhou2022belief}
Qianli Zhou, Yusheng Huang, and Yong Deng.
\newblock Belief evolution network-based probability transformation and fusion.
\newblock {\em Computers \& Industrial Engineering}, 174:108750, 2022.

\end{thebibliography}
\bibliographystyle{plain}

\end{document}